\documentclass[conference]{IEEEtran}
\IEEEoverridecommandlockouts
\usepackage{cite}
\usepackage{amsmath,amssymb,amsfonts,amsthm}
\usepackage{algorithmic}
\usepackage{graphicx}
\usepackage{textcomp}
\usepackage{booktabs}
\usepackage{xcolor}
\usepackage{algorithm}
\usepackage{subcaption}

\newtheorem{theorem}{Theorem}
\newtheorem{lemma}{Lemma}

\newtheorem{assumption}{Assumption}
\newtheorem{remark}{Remark}
\newtheorem{corollary}{Corollary}

\def\BibTeX{{\rm B\kern-.05em{\sc i\kern-.025em b}\kern-.08em
    T\kern-.1667em\lower.7ex\hbox{E}\kern-.125emX}}
\begin{document}

\title{ Communication-Efficient Adam-Type Algorithms for Distributed Data Mining }

\author{\IEEEauthorblockN{Wenhan Xian, Feihu Huang, Heng Huang} 
\IEEEauthorblockA{\textit{Department of Electrical and Computer Engineering, University of Pittsburgh}} 
Pittsburgh, United States \\
wex37@pitt.edu, huangfeihu2018@gmail.com, heng.huang@pitt.edu
}

\maketitle

\begin{abstract}
Distributed data mining is an emerging research topic to effectively and efficiently address hard data mining tasks using big data, which are partitioned and computed on different worker nodes, instead of one centralized server. Nevertheless, distributed learning methods often suffer from the communication bottleneck when the network bandwidth is limited or the size of model is large. To solve this critical issue, many gradient compression methods have been proposed recently to reduce the communication cost for multiple optimization algorithms.
However, the current applications of gradient compression to adaptive gradient method, which is widely adopted because of its excellent performance to train DNNs, do not achieve the same ideal compression rate or convergence rate as Sketched-SGD. To address this limitation, in this paper, we propose a class of novel distributed Adam-type algorithms (\emph{i.e.}, SketchedAMSGrad) utilizing sketching, which is a promising compression technique that reduces the communication cost from $O(d)$ to $O(\log(d))$ where $d$ is the parameter dimension. In our theoretical analysis, we prove that our new algorithm achieves a fast convergence rate of $O(\frac{1}{\sqrt{nT}} + \frac{1}{(k/d)^2 T})$ with the communication cost of $O(k \log(d))$ at each iteration. Compared with single-machine AMSGrad, our algorithm can achieve the linear speedup with respect to the number of workers $n$. 
The experimental results on training various DNNs in distributed paradigm validate the efficiency of our algorithms. 
\end{abstract}


\section{Introduction}

Nowadays, as more and more data mining and machine learning applications take advantage of large-scale data, plenty of learning models are trained in a distributed fashion across many worker nodes \cite{pytorch-distribute}. Specifically, the problem of these tasks can be formulated as:
\begin{align} \label{problem}
    f(x)  = \frac{1}{n} \sum_{i=1}^n \mathbb{E}_{\xi_i \sim D_i} F_i(x; \xi_i), 
\end{align}
where $f_i(x)=\mathbb{E}_{\xi_i \sim D_i} F_i(x; \xi_i)$ is the local objective function on $i$-th node that is generally smooth and possibly nonconvex, and $n$ is the number of worker nodes. Here $D_i$ denotes the data distribution on $i$-th node, and $\{D_i\}_{i=1}^n$ are probably non-identical.

Although distributed training has shown excellent performance and efficiency for solving problem \eqref{problem}, it still suffers from the communication bottleneck, especially when the network bandwidth is limited or the size of model is large. To address the critical high communication cost issue, many methods have been presented to reduce the amount of communication. Among these methods, one of the most popular and common ways is to compress the message to be sent at each communication round, such as gradient quantization \cite{1-bit-sgd, qsgd} and gradient sparsification \cite{sparsification-0, sparsification-2, sparsification}. 

Gradient quantization reduces the communication cost by lowering the float-point precision of gradients so that less amount of bits will be transmitted. 1-bit Stochastic Gradient Descent (1-bit SGD) \cite{1-bit-sgd} is a classic and primitive gradient quantization work which uses 1-bit quantization and dramatically enhances the communication efficiency. Quantized Stochastic Gradient Descent (QSGD) adopts stochastic randomized rounding to obtain an unbiased estimator after compression. SignSGD and its variant with momentum named Signum \cite{signsgd} were designed to only transmit the 1-bit gradient sign between worker and central node, which is convenient to implement. 

Gradient sparsification is another widely-used strategy to decrease the communication cost which sparsifies the gradient instead of quantizing each element. The most popular way is to extract the top-$k$ coordinates of local workers and send them to the master node to estimate the mini-batch gradient. Some of these methods also combine gradient sparsicifation with other techniques such as momentum correction and error-feedback. For example, MEM-SGD \cite{sparsification} adds back the accumulated error before each transmission and is proven to achieve the same convergence rate as SGD. 

Recently, more variants of gradient compression with theoretical guarantees have been proposed, such as SGD with Error-Feedback (EF-SGD) \cite{ef-sgd}, Distributed SGD with Error-Feedback (dist-EF-SGD) \cite{dist-ef-sgd} and SGD with Error Reset (CSER) \cite{NEURIPS2020_94cb02fe}. In some recent works like \cite{NEURIPS2020_94cb02fe, dist-ef-sgd}, the aggregated gradient estimator is also compressed before sending back to workers. Some works also apply gradient compression to other optimizer such as Frank-Wolfe algorithm \cite{Xian_Huang_Huang_2021}.

Besides, to solve problem \eqref{problem}, we also need an efficient optimizer to search for the optimal solution. Among existing popular optimization methods, adaptive gradient algorithms \cite{adagrad, rmsprop} have become ones of the most important optimization algorithms to pursue higher efficiency or accuracy in a wide range of data mining and machine learning problems. In the family of adaptive gradient algorithms, Adam \cite{adam} is one of the most popular ones that combines momentum and adaptive learning rate. Though it achieves great success in practice, several technical issues in the analysis were pointed out \cite{amsgrad} and in some cases the algorithm could diverge.

In \cite{amsgrad}, two variants of Adam, named as AMSGrad and Adamnc, were proposed to fix the theoretical issues in the analysis of Adam. AMSGrad makes quantity {\small $\Gamma_{t+1} = (\frac{\sqrt{V_{t+1}}}{\alpha_{t+1}} - \frac{\sqrt{V_t}}{\alpha_t})$} positive to ensure the convergence, while Adamnc adopts an increasing parameter $\beta_{2,t} = 1 - \frac{1}{t}$. 

Despite of the success of gradient compression methods, it is hard to use them in distributed adaptive gradient method. So far the application of gradient compression to adaptive gradient algorithm with theoretical guarantee is still limited. Quantized Adam \cite{quantized-adam} combines gradient quantization with Adamnc, which keeps track of local momentum and variance terms on each worker node and uses quantization when averaging the parameter. Efficient-Adam \cite{efficient-adam} is similar to Quantized Adam where the gradient message sent back is also compressed. However, both Quantized Adam and Efficient-Adam are not proven to achieve linear speedup or convergence on non-iid data. APMSqueeze \cite{apmsqueeze} and 1-bit Adam \cite{tang20211} are Adam-preconditioned momentum SGD algorithms with gradient compression. However, the variance term is fixed during the training process. Even though it is computed by Adam at the end of warmup step, technically APMSqueeze and 1-bit Adam are not a true adaptive gradient method. 

\emph{Therefore, it is difficult to apply gradient compression to adaptive gradient methods and maintain the excellent performance of distributed Adam-type algorithms. The challenge is that the original adaptive learning rate is adjustable based on global information such as the aggregated gradient. Although the compressed message is a good estimation of local gradient or momentum, the adaptive learning rate calculated by these inexact messages could be far away from the original one.} 

To address the challenging high communication cost limitation in distributed adaptive gradient methods, we propose a class of novel distributed Adam-type algorithms (called as SketchedAMSGrad), based on the distributed version of AMSGrad \cite{amsgrad} algorithm and the gradient sparsification technique named sketching \cite{sketched-sgd, privacy-sketch}.

Our main contributions are summarized as follows:
\begin{itemize}
    \item [(1)] To efficiently address the communication bottleneck problem in distributed data mining, we propose a class of novel communication-efficient algorithms named SketchedAMSGrad with two averaging strategies: parameter averaging and gradient averaging. Our new methods can reduce the communication cost from $O(d)$ to $O(\log(d))$. 
    \item [(2)] We provide theoretical analysis based on mild assumptions to guarantee the convergence of our algorithms. Specifically, we prove that our SketchedAMSGrad algorithms have a convergence rate of $O(\frac{1}{\sqrt{nT}})$, which shows a linear speedup. Our theoretical analysis also allows the data distribution to be non-identical.
    \item [(3)] To the best of our knowledge, our method is the first one to utilize the sketching technique to solve the communication bottleneck in distributed adaptive gradient methods. The experimental results on training various DNNs verify the performances of our algorithms, on both identical and non-identical distributed datasets.
\end{itemize}

\begin{table*}
\setlength{\tabcolsep}{10pt}
  \caption{Comparison of Related Algorithms with Compression}
  \label{related-table}
  \centering
  \begin{tabular}{cccccc}
    \toprule
    Name     & Convergence rate     & Linear speedup   & Non-iid   & Adaptive    & Reference \\
    \midrule
    Quantized-Adam & $O(\frac{1}{\sqrt{T}})$  & $\times$  & $\times$  & $\surd$   & \cite{quantized-adam}   \\
    Efficient-Adam & $O(\frac{1}{\sqrt{T}})$  & $\times$  & $\times$  & $\surd$   & \cite{efficient-adam}   \\
    APMSqueeze     & $O(\frac{1}{\sqrt{nT}} + \frac{1}{(k/d)^{2/3} T^{2/3}})$ & $\surd$  & $\surd$  & $\times$   & \cite{apmsqueeze}      \\
    1-bit Adam     & $O(\frac{1}{\sqrt{nT}} + \frac{1}{(k/d)^{2/3} T^{2/3}})$ & $\surd$  & $\surd$  & $\times$   & \cite{tang20211}      \\
    SketchedAMSGrad (GA) & $O(\frac{1}{\sqrt{nT}} + \frac{1}{(k/d)^2 T})$  & $\surd$  & $\surd$  & $\surd$   & this paper   \\
    \bottomrule
  \end{tabular}
  \vspace{-10pt}
\end{table*}

\section{Related Works} \label{related}
In the section, we review the related adaptive gradient algorithms with their compressed versions and introduce some preliminary background of sketching. The summary of properties of related methods is listed in Table \ref{related-table}. Top-$k$ is considered as the compressor in the result of convergence rate.
\subsection{ Quantized-Adam and Efficient-Adam }
Quantized-Adam \cite{quantized-adam} is proposed to combine quantization scheme with distributed Adam algorithm to reduce the communication cost. Specifically, on each worker, it owns a local momentum term $m_t^{(i)}$ and a local variance term $v_t^{(i)}$. These two terms are updated by the exponential moving averaging used in Adam-type algorithms. Gradient quantization is used to compress the term $m_t^{(i)} / \sqrt{v_t^{(i)}}$.

Efficient-Adam \cite{efficient-adam} is a similar work to Quantized Adam. The only difference is that when the parameter server sends information back to worker nodes, Efficient-Adam compresses the updating term, which is more common in related works, while Quantized-Adam quantizes the parameter. Actually, both Quantized-Adam and Efficient-Adam allow other compressors if they satisfy the compressor assumption that there exists a constant $\theta \in (0,1]$ such that
\begin{equation} \label{compressor}
    \lVert C(x) - x \rVert \le (1 - \theta) \lVert x \rVert. 
\end{equation}

These two algorithms are parameter averaging since if there is no compression, they degenerate to an algorithm where each node is updated by Adam and then the model parameter is averaged. It is not mathematically equivalent to the typical distributed Adam algorithm where gradient averaging is used. Though in some cases parameter averaging is convenient to implement, it is likely to cause bad convergence or be detrimental to the model accuracy especially when the optimizer relies on past local gradient \cite{pytorch-distribute}. Besides, in the convergence analysis of Quantized-Adam and Efficient-Adam, the data distribution $\{D_i\}_{i=1}^n$ have to be identical and the convergence rate does not achieve a linear speedup.
\subsection{ APMSqueeze and 1-bit Adam Algorithms }
APMSqueeze \cite{apmsqueeze} and 1-bit Adam \cite{tang20211} are communication-efficient Adam-preconditioned momentum SGD algorithms. Since the definitions of these two algorithms are similar and 1-bit Adam is the later work, in this paper we will only discuss 1-bit Adam. In the warmup stage, it calculates a variance term $v_{T_w}$. During the training process, $v_{T_w}$ is fixed and serves as the exponential moving averages term $v_t$ in regular Adam-type algorithms. 1-bit Adam is a gradient averaging algorithm. According to the convergence analysis of \cite{tang20211}, 1-bit Adam achieves a linear speedup with a convergence rate of $O(\frac{1}{\sqrt{nT}})$ for a fixed $T$. However, since $v_{T_w}$ is a fixed variable, 1-bit Adam is not technically an adaptive gradient method. In our method, the variance term $v_t$ is dynamic and computed by exponential moving averaging. Besides, we do not need a separate warmup stage where another communication-inefficient optimizer is used. Furthermore, 1-bit Adam requires the gradient compressor to satisfy an assumption that $\lVert C(x) - x \rVert \le \epsilon$ for some constant $\epsilon$. Under the condition of this paper, the bound $epsilon$ in 1-bit Adam should be $O(\frac{dG}{k})$ where $G$ is the bounded gradient. Hence we can reach the convergence rate in Table~\ref{related-table}. The second dominating term in the convergence rate is $O(\frac{1}{(k/d)^{2/3} T^{2/3}})$. We will compare it with the result of sketching method in next subsection.

\subsection{Sketching}
In this subsection we introduce some preliminary background about sketching before moving forward to our proposed algorithms. Sketching \cite{sketched-sgd,privacy-sketch} is a novel and promising gradient sparsicifation technique that compresses a gradient vector $g$ into a sketch $S(g)$ of size $O(\log(d)\epsilon^{-1})$ such that $S(g)$ can approximately recover every coordinates by $\hat{g}_i^2 = g_i^2 \pm \epsilon \lVert g \rVert_2^2$. It is originated from a data structure used in data streaming named Count Sketch \cite{count-sketch} which is designed to find large coordinates in a vector $g$ defined by a sequence of updates $\{(i_j, w_j)\}_{j=1}^n$. When we use sketching to reduce the communication cost, the sketching and unsketching process are demonstrated in Algorithm 4 in \cite{sketched-sgd}. We have a $r\times c$ table of counters $S$, sign hashes $\{ s_j \}_{j=1}^r$ and bucket hashes $\{ h_j \}_{j=1}^r$. Given an update $(i, f_i)$, where $i$ is an index and $f_i$ is the $i$-th coordinate of a vector $f$, $S$ is updated by $S[j, h_j(i)] \mathrel{+}= s_j(i) f_i$ for $j=1, \cdots r$. It is obvious that sketching operator is linear and satisfies the following formulation:  
\begin{equation}
    S(\alpha g_1 + \beta g_2) = \alpha S(g_1) + \beta S(g_2). \vspace{-1pt}
\end{equation}
Therefore, the sketches from different workers can be aggregated on the parameter server. The unsketching operator is to get an estimation which is derived from the median value of $s_j(i) S[j, h_j(i)]$ for $j=1, \cdots r$.

In \cite{sketched-sgd}, sketching serves as a compressor that will approximately recover the true top-$k$ coordinates of mini-batch gradient $\frac{1}{n} \sum_{i=1}^n g_t^{(i)}$ where $n$ is the number of workers. In \cite{fetch-sgd}, the authors explicitly treat it as a compressor and denote the sketching and unsketching operators by $\mathcal{S}$ and $\mathcal{U}$ respectively. For convenience, we also use these notations in this paper. Sketching method reduces the communication cost to $O(\log(d))$ while gradient quantization only achieves a constant level reduction and the communication cost is still $O(d)$. The current best results for quantization method achieve an approximate $32\times$ compression rate \cite{signsgd-majority, dist-ef-sgd}. vqSGD \cite{vqsgd} is actually a sparsification method that maps a gradient vector to the set of vertices of a convex hull so here we do not categorize it as a quantization method. Compared with top-$k$ method, one advantage of sketching is to recover the true top-$k$ coordinates, where the gradient estimator is $v_1 \approx Top_k (\frac{1}{n} \sum_{i=1}^n g_t^{(i)})$. Although applying the method in \cite{dist-ef-sgd} can avoid the $O(n)$ return communication cost mentioned in \cite{sketched-sgd}, the gradient estimator $v_2 = Top_k (\frac{1}{n} \sum_{i=1}^n Top_k(g_t^{(i)}))$ is still probably far away from the true top-$k$ coordinates. This issue can be reflected by the second dominating term in the convergence rate. In \cite{dist-ef-sgd}, the second dominating term is $O(\frac{1}{(k/d)^{4/3} T^{2/3}})$ which is claimed to be the price to pay for two-way compression and linear speedup. In 1-bit Adam the step size is dependent on the compression ratio and this term becomes $O(\frac{1}{(k/d)^{2/3} T^{2/3}})$ as we have mentioned. However, in Sketched-SGD and our algorithms, the corresponding term is $O(\frac{1}{(k/d)^2 T})$, which is smaller when $T$ is large.

\section{Sketched Adam-type Algorithms} \label{algorithm}

In the section, we propose a class of efficient sketched distributed Adam-type algorithms. 
\subsection{SketchedAMSGrad (Parameter Averaging)}

In this subsection, we will propose the SketchedAMSGrad (PA) algorithm using parameter averaging, the description of which is shown in Algorithm \ref{alg:parameter-average}. 

In Algorithm \ref{alg:parameter-average}, we use AMSGrad algorithm to update each worker node, based on local momentum term $m_t^{(i)}$ and exponential moving averages of squared past gradients $v_t^{(i)}$. $\alpha_t$ is the stepsize and $\beta_1, \beta_2 \in (0, 1)$ are exponential moving average hyperparameters in Adam-type algorithm. $\epsilon > 0$ is the initial value of $v_0$ to avoid zero denominators. The multiplication, division and square operation between vectors are component-wise. We use sketching to improve communication efficiency and average the parameters. We also use error-feedback to further accelerate the convergence. 

For convenience, we also use the notations $\mathcal{S}$ and $\mathcal{U}$ defined in \cite{fetch-sgd} to represent the sketching operator and unsketching operator. They can be treated as a compressor that will approximately recover the true top-$k$ coordinates. In practice, we use a second round communication which is also required in SketchedSGD \cite{sketched-sgd}. After unsketching, we get an estimation of the aggregated mini-batch gradient which is denoted by $\mathcal{U}(S_t)$. Then we select the largest $Pk$ coordinates to extract their exact values before sketching from each worker during the second round communication. Finally, we select the top-$k$ coordinates among these $Pk$ coordinates as $\Delta_t$ and send it back to each worker. $\Delta_t^{(i)}$ contains the corresponding $k$ coordinates in $m_t^{(i)}/ \sqrt{\hat{v}_t^{(i)}} + \frac{\alpha_{t-1}}{\alpha_t} e_{t-1}^{(i)}$ and automatically it satisfies $\Delta_t = \frac{1}{n} \sum_{i=1}^n \Delta_t^{(i)}$. Therefore, at each iteration, the total communication cost is $|S| + Pk + k$ and the compression rate is $2d/(|S| + Pk + k)$ where $|S|$ is the size of sketch.   
Using lemma 1 in \cite{sketched-sgd} and replacing $\tilde{g}_t$ and $\bar{g}_t^i$ with $\Delta_t$ and $(m_t^{(i)}/ \sqrt{\hat{v}_t^{(i)}} + \frac{\alpha_{t-1}}{\alpha_t} e_{t-1}^{(i)})$, we can obtain the following Lemma~\ref{lemma1}.
\begin{lemma} \label{lemma1}
In Algorithm~\ref{alg:parameter-average}, let $\tilde{\Delta}_t \!=\! \frac{1}{n} \sum_{i=1}^n (m_t^{(i)}/ \sqrt{\hat{v}_t^{(i)}} + \frac{\alpha_{t-1}}{\alpha_t} e_{t-1}^{(i)})$, and give sketch size $\Theta(k \log(d/\delta))$, with the probability $\ge 1 - \delta$, we have
\vspace{-6pt}
\begin{equation}
    \lVert \Delta_t - \tilde{\Delta}_t \rVert^2 \le (1 - \frac{k}{d}) \lVert \tilde{\Delta}_t \rVert^2.
\end{equation}
\end{lemma}
Lemma~\ref{lemma1} indicates that $\Delta_t$ is an estimation of $\tilde{\Delta}_t$ and illustrates how sketch can serve as a compressor.
\setlength{\textfloatsep}{12pt}
\begin{algorithm}[t]
   \caption{SketchedAMSGrad (parameter averaging) }
   \label{alg:parameter-average}
\begin{algorithmic}
   \STATE {\bfseries Input:} initial value $x_1$, sketching operator $\mathcal{S}$ and unsketching operator $\mathcal{U}$
   \STATE {\bfseries Set:} $m_0^{(i)} \!=\! \mathbf{0}$, $v_0^{(i)} \!=\! \hat{v}_0^{(i)} \!=\! \mathbf{\epsilon}$, $e_0^{(i)} \!=\! \mathbf{0}$ on $i$-th worker node
   \FOR{$t=1$ {\bfseries to} $T$}
   \STATE On $i$-th worker node:
   \STATE\hspace{1.2em} Estimate a stochastic gradient $g_t^{(i)}$; 
   \STATE\hspace{1.2em} Compute $m_t^{(i)} = \beta_1 m_{t-1}^{(i)} + (1 - \beta_1) g_t^{(i)}$;
   \STATE\hspace{1.2em} $v_t^{(i)} = \beta_2 v_{t-1}^{(i)} + (1 - \beta_2) [g_t^{(i)}]^2$; 
   \STATE\hspace{1.2em} $\hat{v}_t^{(i)} = \max \{\hat{v}_{t-1}^{(i)}, v_t^{(i)} \}$; 
   \STATE\hspace{1.2em} Sketch $S_t^{(i)} = \mathcal{S}(m_t^{(i)}/ \sqrt{\hat{v}_t^{(i)}} + \frac{\alpha_{t-1}}{\alpha_t} e_{t-1}^{(i)})$; 
   \STATE\hspace{1.2em} Send $S_t^{(i)}$ to the master node; 
   \STATE\hspace{1.2em} Send $\Delta_t^{(i)}$ to the master node after unsketching; 
   \STATE\hspace{1.2em} Compute $e_t^{(i)} = m_t^{(i)}/ \sqrt{\hat{v}_t^{(i)}} + \frac{\alpha_{t-1}}{\alpha_t} e_{t-1}^{(i)} - \Delta_t^{(i)}$; 
   \STATE\hspace{1.2em} Receive $\Delta_t$ from the master node; 
   \STATE\hspace{1.2em} Update $x_{t+1} = x_t - \alpha_t \Delta_t$. 
   \STATE On the master node:
   \STATE\hspace{1.2em} Aggregate $S_t = \frac{1}{n} \sum_{i=1}^n S_t^{(i)}$;
   \STATE\hspace{1.2em} Unsketch $\Delta_t = \frac{1}{n} \sum_{i=1}^n \Delta_t^{(i)} = \mbox{Top-$k$}(\mathcal{U}(S_t))$;
   \STATE\hspace{1.2em} Send $\Delta_t$ back to each worker node; 
   \STATE\hspace{1.2em} Update $x_{t+1} = x_t - \alpha_t \Delta_t$. 
   \ENDFOR
\end{algorithmic}
\end{algorithm}

\begin{algorithm}[tb]
   \caption{SketchedAMSGrad (gradient averaging) }
   \label{alg:gradient-average}
\begin{algorithmic}
   \STATE {\bfseries Input:} initial value $x_1$, sketching operator $\mathcal{S}$ and unsketching operator $\mathcal{U}$
   \STATE {\bfseries Set:} $m_0^{(i)} = \mathbf{0}$, $e_0^{(i)} = \mathbf{0}$ on $i$-th worker node; $v_0 = \hat{v}_0$ on the master node; index set $\mathcal{I}_0 = \varnothing$
   \FOR{$t=1$ {\bfseries to} $T$}
   \STATE On $i$-th worker node:
   \STATE\hspace{1.2em} Estimate a stochastic gradient $g_t^{(i)}$; 
   \STATE\hspace{1.2em} Compute $m_t^{(i)} = \beta_1 m_{t-1}^{(i)} + (1 - \beta_1) g_t^{(i)}$; 
   \STATE\hspace{1.2em} Send $h_t^{(i)} = (g_t^{(i)})_{\mathcal{I}_{t-1}}$ to the master node;
   \STATE\hspace{1.2em} Sketch $S_t^{(i)} = \mathcal{S}(m_t^{(i)} + \frac{\alpha_{t-1}}{\alpha_t} e_{t-1}^{(i)})$; 
   \STATE\hspace{1.2em} Send $S_t^{(i)}$ to the master node; 
   \STATE\hspace{1.2em} Send $\Delta_t^{(i)}$ to the master node after unsketching; 
   \STATE\hspace{1.2em} Compute $e_t^{(i)} = m_t^{(i)} + \frac{\alpha_{t-1}}{\alpha_t} e_{t-1}^{(i)} - \Delta_t^{(i)}$;
   \STATE\hspace{1.2em} Receive $\Delta_t$ from the master node; 
   \STATE\hspace{1.2em} Update $x_{t+1} = x_t - \alpha_t \Delta_t$.
   \STATE On the master node:
   \STATE\hspace{1.2em} Aggregate $h_t = \frac{1}{n} \sum_{i=1}^n h_t^{(i)}$;
   \STATE\hspace{1.2em} Compute $v_t = \beta_2 v_{t-1} + (1 - \beta_2)  h_t^2$; 
   \STATE\hspace{1.2em} $\hat{v}_t = \max \{\hat{v}_{t-1}, v_t \}$;
   \STATE\hspace{1.2em} Aggregate $S_t = \frac{1}{n} \sum_{i=1}^n S_t^{(i)}$; 
   \STATE\hspace{1.2em} Unsketch $\Delta_t \!=\! \frac{1}{n} \sum_{i=1}^n \Delta_t^{(i)} \!=\! \mbox{Top-$k$}(\mathcal{U}(S_t, \hat{v}_t))$;
   \STATE\hspace{1.2em} Send $\Delta_t$ back to each worker node; 
   \STATE\hspace{1.2em} Update $x_{t+1} = x_t - \alpha_t \Delta_t$.
   \ENDFOR
\end{algorithmic}
\end{algorithm}
\subsection{SketchedAMSGrad (Gradient Averaging)}
In the subsection, we propose 
the SketchedAMSGrad (GA) algorithm using gradient averaging, which is demonstrated in Algorithm~\ref{alg:gradient-average}. 

In Algorithm~\ref{alg:gradient-average}, the meanings of hyperparameters $\alpha_t$, $\beta_1$ and $\beta_2$ are the same as those in Algorithm~\ref{alg:parameter-average}. We also keep track of local momentum term $m_t^{(i)}$ on each node but the exponential moving averaging squared gradient $v_t$ is defined on the master node. The index set $\mathcal{I}_t$ represents the coordinates updated at iteration $t$, which is obtained by the unsketching operator. Notation $h_t^{(i)} = (g_t^{(i)})_{\mathcal{I}_{t-1}}$ means for $\forall j \in \mathcal{I}_{t-1}$, $h_t^{(i)}$ maintains the $j$-th coordinate of $g_t^{(i)}$. Otherwise, if $j \notin \mathcal{I}_{t-1}$, the $j$-th coordinate of $h_t^{(i)}$ is $0$. We define $\mathcal{I}_t$ in this way because we want to accumulate the coordinates of squared gradient which are just updated and we want to define an auxiliary sequence that makes the convergence analysis more convenient. Algorithm~\ref{alg:gradient-average} is a gradient averaging algorithm because if there is no compressor applied, this algorithm is degenerated to the common distributed AMSGrad optimizer. In Algorithm~\ref{alg:gradient-average} the unsketching operator $\mathcal{U}$ requires a vector $\hat{v}_t$ as another input and is used to recover the top-$k$ coordinates of term $\tilde{\Delta}_t$, which is defined as follows.
\begin{equation}
    \tilde{\Delta}_t = \frac{1}{n} \sum_{i=1}^n \tilde{\Delta}_t^{(i)}, \  \tilde{\Delta}_t^{(i)} = \hat{v}_t^{-1/2} (m_t^{(i)} + \frac{\alpha_{t-1}}{\alpha_t} e_{t-1}^{(i)})
\end{equation}
The index set of these $k$ coordinates is denoted as $\mathcal{I}_t$. The implementation of $\mathcal{U}$ is shown in Algorithm~\ref{alg:unsketching}, which is established on the original sketching and unsketching operator. According to the linear property of sketching $\mathcal{S}$, it is equivalent to compress $\tilde{\Delta}_t$ by $\mathcal{S}$ and then unsketch it by the normal unsketching operator. $\Delta_t^{(i)}$ contains the coordinates of $\tilde{\Delta}_t^{(i)}$ that belongs to index set $\mathcal{I}_t$ and $\Delta_t = \frac{1}{n} \sum_{i=1}^n \Delta_t^{(i)}$. Therefore, using lemma 1 in \cite{sketched-sgd} and replacing $\tilde{g}_t$ and $\bar{g}_t^i$ with $\Delta_t$ and $\tilde{\Delta}_t^{(i)}$, we reach our following Lemma~\ref{lemma2}.
\begin{algorithm}[tb]
\caption{Unsketching Operator in Algorithm~\ref{alg:gradient-average}}
\begin{algorithmic}
    \STATE {\bfseries Input:} $r\times c$ sketch $S$, vector $v$, bucket hashes $\{ h_j \}_{j=1}^r$, original unsketching operator $\mathcal{U}_0$
    \FOR{$i=1$ {\bfseries to} $d$}
    \FOR{$j=1$ {\bfseries to} $r$}
    \STATE $S[j, h_j(i)] = S[j, h_j(i)] / \sqrt{v_i}$
    \ENDFOR
    \ENDFOR
    \STATE return $\mathcal{U}_0(S)$
\end{algorithmic}
\label{alg:unsketching}
\end{algorithm}
\begin{lemma}
\label{lemma2}
With sketch size $\Theta(k \log(d/\delta))$ and with probability $\ge 1 - \delta$ in Algorithm~\ref{alg:gradient-average}, we have
\begin{equation}
    \lVert \Delta_t - \tilde{\Delta}_t \rVert^2 \le (1 - \frac{k}{d}) \lVert \tilde{\Delta}_t \rVert^2
\end{equation}
\end{lemma}
Lemma~\ref{lemma2} is the key lemma to the analysis of our Algorithm~\ref{alg:gradient-average} which provides an estimation of term $m_t/\sqrt{\hat{v}_t}$. It is also the motivation to apply sketching in communication efficient Adam-type algorithms. As the top-$k$ coordinates of $m_t/\sqrt{\hat{v}_t}$ and $m_t$ are likely to change a lot, it is hard to estimate the Adam updating term $m_t/\sqrt{\hat{v}_t}$ by the known vector $m_t^{(i)}$ on each node. However, the sketching technique makes it possible within the communication cost of $O(\log (d))$.

In Algorithm~\ref{alg:gradient-average}, thus, the total communication cost at each iteration is $|S| + Pk + 2k$ and the compression rate is $2d/(|S| + Pk + 2k)$ where $|S|$ is the size of sketch. 

In fact, our SketchedAMSGrad (GA) algorithm is compatible with 1-bit Adam algorithm. We can also regard the $v_{T_w}$ in the 1-bit Adam algorithm as the initial value of $v_0$ in Algorithm~{\ref{alg:gradient-average}}. The only difference is that in the theoretical analysis we need to replace the initial value $\epsilon$ with the $v_{min}$ defined in the 1-bit Adam. Moreover, if we do not send $h_t$ or update $v_t$, our algorithm is reduced to the 1-bit Adam with sketching compressor.
\section{Convergence Analysis} \label{analysis}

In the section, we provide the convergence analysis of our algorithms. Due to the space limit, we will only provide the proof outline for Theorem \ref{thm1} and full proof for Theorem \ref{thm2}. We begin with giving some mild assumptions. 

\begin{assumption}
 (Lipschitz Gradient) There is a constant $L$ such that for $\forall x, y \in R^d$, $\lVert \nabla f(x) - \nabla f(y) \rVert \le L \lVert x - y \rVert$. 
\end{assumption}
\begin{assumption}
(Lower Bound) 
Function $f(x)$ has the lower bound, i.e., $\inf_{x\in R^d} f(x) = f^* > - \infty$
\end{assumption}
\begin{assumption}
(Bounded Gradient)
There is a constant $G$ such that for $\forall i \in \{1, \cdots, n\}$, $\forall \xi_i \sim D_i$, we have
$ \lVert \nabla F_i(x; \xi_i) \rVert_{\infty} \le G$. 
\end{assumption}
These assumptions are commonly used in related works of Adam-type algorithms in nonconvex optimization \cite{nonconvex-adam-2, nonconvex-adam-1, new-regret-adam}. In our convergence analysis, we define the following constants.
\begin{equation}
    \gamma_0 = (1 - \delta) (1 - \frac{k}{d}) + \delta,\ \gamma = 1 - \frac{k}{2d}(1 - \delta),\ \gamma_1 = \frac{(3 - \gamma_0)\gamma_0}{1 - \gamma_0}
\end{equation}
\subsection{SketchedAMSGrad (PA) }
\begin{theorem}
\label{thm1}
Assume that Assumption 1 to Assumption 3 are satisfied and data distribution $\{D_i\}_{i=1}^n$ are identical. In Algorithm~\ref{alg:parameter-average}, let $\beta_1 < 1$, $\beta_2 < 1$, $\epsilon > 0$ and $\alpha_t = \frac{\alpha}{\sqrt{1+T}}$, $\alpha > 0$. Then we have 
\begin{equation}
    \nonumber \frac{1}{T} \sum_{t=1}^T \mathbb{E} \lVert \nabla f(x_t) \rVert^2 \le \frac{C_1}{\sqrt{T}} + \frac{C_2}{T},
\end{equation}
where constants $C_1$ and $C_2$ are independent of $T$.
\end{theorem}
To prove Theorem~\ref{thm1}, we define a useful auxiliary sequence $\tilde{x}_t$ such that $\tilde{x}_1 = x_1$ and 
\begin{align} \label{eq:13}
    \tilde{x}_{t+1} = \tilde{x}_t - \alpha_t \frac{1}{n} \sum_{i=1}^n m_t^{(i)} / \sqrt{\hat{v}_t^{(i)}}.
\end{align}
Let $e_t = \frac{1}{n} \sum_{i=1}^n e_t^{(i)}$. The error compensation term $e_t^{(i)}$ is multiplied by a factor $\alpha_{t-1}/\alpha_t$ because it always satisfies 
\begin{equation}
\label{aux-diff}
    x_t - \tilde{x}_t = \alpha_{t-1} e_{t-1}.
\end{equation}
Next we will provide the proof outline of Theorem~\ref{thm1}.
\begin{proof}
Let $A_t^{(i)} = \alpha_t [\hat{v}_t^{(i)}]^{-1/2} \nabla f(\tilde{x}_t)$ for $i = 1,\cdots, n$, $t = 1, \cdots, T$ and $A_0^{(i)} = A_1^{(i)}$. Since $m_t^{(i)} = \beta_1 m_{t-1}^{(i)} + (1 - \beta_1) g_t^{(i)}$ and $m_0^{(i)} = \mathbf{0}$, it is easy to check the following equation:
\begingroup
\small
\begin{align}
\label{outline1}
    \sum_{t=1}^T \langle A_t^{(i)}, g_t^{(i)} \rangle &= \frac{\beta_1}{1 - \beta_1} \langle A_T^{(i)}, m_T^{(i)} \rangle + \sum_{t=1}^T \langle A_t^{(i)}, m_t^{(i)} \rangle \notag \\
    &\quad + \frac{\beta_1}{1 - \beta_1} \sum_{t=1}^T \langle A_t^{(i)} - A_{t+1}^{(i)}, m_t^{(i)} \rangle
\end{align}
\endgroup
The left hand side of Eq.~(\ref{outline1}) can be rewritten by
\begingroup
\small
\begin{align}
\label{outline2}
    \langle A_t^{(i)}, g_t^{(i)} \rangle &= \langle \alpha_{t-1} [\hat{v}_{t-1}^{(i)}]^{-1/2} \nabla f(x_t), g_t^{(i)} \rangle - \langle (\alpha_{t-1} [\hat{v}_{t-1}^{(i)}]^{-1/2} \notag \\
    & \quad -\! \alpha_t [\hat{v}_t^{(i)}]^{-1/2}) \nabla f(\tilde{x}_t), g_t^{(i)} \rangle \notag \\
    & \quad -\! \langle \alpha_t [\hat{v}_{t-1}^{(i)}]^{-1/2} (\nabla f(x_t) \!-\! \nabla f(\tilde{x}_t)), g_t^{(i)} \rangle
\end{align}
\endgroup
Let $\xi_t^{(i)}$ be the sample index set at iteration $t$ on node $i$. As data distribution $D_i$'s are identical, we have 
\begin{equation}
    \mathbb{E}_{\xi_t^{(i)}} g_t^{(i)} = \nabla f(x_t)
\end{equation}
Therefore, if taking expectation on $\langle  [\hat{v}_{t-1}^{(i)}]^{-1/2} \nabla f(x_t), g_t^{(i)} \rangle$ over $\xi_t^{(i)}$, we can replace the $g_t^{(i)}$ with $\nabla f(x_t)$. But this operation is not allowed on $\langle A_t^{(i)}, g_t^{(i)} \rangle$ because $\hat{v}_t^{(i)}$ is also dependent on $\xi_t^{(i)}$. We deal with it in this way because the previous value $\hat{v}_{t-1}^{(i)}$ is not determined by $\xi_t^{(i)}$. Next we will estimate the terms in Eq.~(\ref{outline1}) and (\ref{outline2}). Using Assumption 1 and the definition of $\tilde{x}_t$ in \eqref{eq:13}, we have
\begingroup
\small
\begin{align}
\label{OL1T2}
\nonumber & \quad \frac{1}{n} \sum_{i=1}^n \langle A_t^{(i)}, m_t^{(i)} \rangle = \langle \nabla f(\tilde{x}_t), \frac{1}{n} \sum_{i=1}^n \alpha_t [\hat{v}_t^{(i)}]^{-1/2} m_t^{(i)} \rangle \\
&\le f(\tilde{x}_t) - f(\tilde{x}_{t+1}) + \frac{L}{2} \lVert \tilde{x}_{t+1} - \tilde{x}_t \rVert^2
\end{align}
\endgroup
By Young's inequality and Assumption 3 we can obtain 
\begingroup
\small
\begin{align}
\label{OL1T1}
\frac{1}{n} \sum_{i=1}^n \langle A_T^{(i)}, m_T^{(i)} \rangle \le L \lVert \tilde{x}_{T+1} - \tilde{x}_T \rVert^2 + \frac{G^2d}{4L}
\end{align}
\endgroup
With Assumption 3 and $\hat{v}_{t+1}^{(i)} \ge \hat{v}_t^{(i)}$, we can also obtain
\begingroup
\small
\begin{align}
\label{OL1T3-i}
\nonumber \langle A_t^{(i)} \!&-\! A_{t+1}^{(i)}, m_t^{(i)} \rangle \le G^2(\lVert \alpha_t [\hat{v}_t^{(i)}]^{-1/2} \rVert_1 \!-\! \lVert \alpha_{t+1} [\hat{v}_{t+1}^{(i)}]^{-1/2} \rVert_1) \\
& + \langle \nabla f(\tilde{x}_t) - \nabla f(\tilde{x}_{t+1}), \alpha_t [\hat{v}_t^{(i)}]^{-1/2} m_t^{(i)} \rangle
\end{align}
\endgroup
Sum $i$ from $1$ to $n$ on Eq.~(\ref{OL1T3-i}) and we have
\begingroup
\small
\begin{align}
\label{OL1T3}
\nonumber & \frac{1}{n} \sum_{i=1}^n \langle A_t^{(i)} - A_{t+1}^{(i)}, m_t^{(i)} \rangle \le L \lVert \tilde{x}_{t+1} - \tilde{x}_t \rVert^2 \\
& \quad + \frac{G^2}{n} \sum_{i=1}^n (\lVert \alpha_t [\hat{v}_t^{(i)}]^{-1/2} \rVert_1 - \lVert \alpha_{t+1} [\hat{v}_{t+1}^{(i)}]^{-1/2} \rVert_1)
\end{align}
\endgroup
where Assumption 1 is used. According to Assumption 3 and the ascent of $\hat{v}_t^{(i)} \ge \hat{v}_{t-1}^{(i)}$, we have
\begingroup
\begin{align}
\label{OL2T2}
    & \quad \langle (\alpha_{t-1} [\hat{v}_{t-1}^{(i)}]^{-1/2} - \alpha_t [\hat{v}_t^{(i)}]^{-1/2}) \nabla f(\tilde{x}_t), g_t^{(i)} \rangle \notag \\
    &\le G^2(\lVert \alpha_{t-1} [\hat{v}_{t-1}^{(i)}]^{-1/2} \rVert_1 - \lVert \alpha_t [\hat{v}_t^{(i)}]^{-1/2} \rVert_1)
\end{align}
\endgroup
Next we can bound the last term in Eq.~(\ref{outline2}).
\begingroup
\begin{align}
\label{OL2T3-b}
    \nonumber & \quad \mathbb{E} \langle \alpha_t [\hat{v}_{t-1}^{(i)}]^{-1/2} (\nabla f(x_t) - \nabla f(\tilde{x}_t)), g_t^{(i)} \rangle \le \\
    & \frac{1}{2} \mathbb{E} \langle \alpha_t [\hat{v}_{t-1}^{(i)}]^{-1/2} \nabla f(x_t), g_t^{(i)} \rangle + \frac{\alpha_t L^2}{2 \sqrt{\epsilon}} \mathbb{E} \| x_t - \tilde{x}_t \|^2
\end{align}
\endgroup
where we have used Assumption 1 and Cauchy-Schwartz inequality. The first term of Eq.~(\ref{OL2T3-b}) can be merge into the first term of Eq.~(\ref{outline2}). To finish the proof, we only need to estimate $\mathbb{E} \| x_t - \tilde{x}_t \|^2$. It can be estimated according to the following inequality
\begingroup
\begin{align}
\label{err2}
\nonumber \mathbb{E} \lVert \alpha_t e_t \rVert^2 &\le \gamma_0 \mathbb{E} \lVert \frac{1}{n} \sum_{i=1}^n \alpha_t m_t^{(i)} / \sqrt{\hat{v}_t^{(i)}} + \alpha_{t-1} e_{t-1} \rVert^2 \\
\nonumber &\le \gamma_1 \mathbb{E} \lVert \frac{1}{n} \sum_{i=1}^n \alpha_t m_t^{(i)} / \sqrt{\hat{v}_t^{(i)}} \rVert^2 + \gamma \mathbb{E} \lVert \alpha_{t-1} e_{t-1} \rVert^2 \\
&\le \gamma_1 \sum_{s=1}^t \gamma^{t-s} \mathbb{E} \lVert \frac{1}{n} \sum_{i=1}^n \alpha_s [\hat{v}_s^{(i)}]^{-1/2} m_s^{(i)} \rVert^2 \notag \\
&= \gamma_1 \sum_{s=1}^t \gamma^{t-s} \mathbb{E} \| \tilde{x}_{s+1} - \tilde{x}_s \|^2
\end{align}
\endgroup
Here the first inequality is because with probability $p > 1 - \delta$, it satisfies $\lVert \alpha_t e_t \rVert^2 \le (1 - \frac{k}{d})\ \lVert \frac{1}{n} \sum_{i=1}^n \alpha_t [\hat{v}_{t}^{(i)}]^{-1/2} m_t^{(i)} + \alpha_{t-1} e_{t-1} \rVert^2$. Otherwise with probability $p < \delta$, $\Delta_t$ is still some coordinates of $\tilde{\Delta}_t$. It always satisfies $\lVert \alpha_t e_t \rVert \le \lVert \frac{1}{n} \sum_{i=1}^n \alpha_t [\hat{v}_{t}^{(i)}]^{-1/2} m_t^{(i)} + \alpha_{t-1} e_{t-1} \rVert$. Hence we can get the first inequality of Eq.~(\ref{err2}). In the third inequality of Eq.~(\ref{err2}) we use Young's inequality. In the third inequality of Eq.~(\ref{err2}) we apply recursion to the second inequality.

Finally, we can reach the conclusion of Theorem \ref{thm1} with the following $C_1$ and $C_2$
\begin{align}
C_1 &= \frac{2G(f(x_1) - f^*)}{\alpha} + \frac{\beta_1 G^3 d}{2L\alpha (1 - \beta_1)} + \frac{4 GL \alpha d}{(1 - \beta_1)(1 - \beta_2)}, \notag \\
C_2 &= \frac{GL^2 \alpha^2 d \gamma_1}{\sqrt{\epsilon} (1 - \beta_2)(1 - \gamma)} + \frac{2G^3d}{\sqrt{\epsilon} (1 - \beta_1)} .
\end{align}
which will be omitted due to space limit.
\end{proof}

\vspace{-5pt}
\subsection{SketchedAMSGrad (GA) }
\begin{theorem}
\label{thm2}
Assume that Assumptions 1-3 are satisfied. In Algorithm \ref{alg:gradient-average}, let $\beta_1 < 1$, $\beta_2 < 1$, $\epsilon > 0$ and $\alpha_t = \frac{\alpha}{\sqrt{1+T/n}}$, $\alpha > 0$. Then we have 
\begin{equation}
    \nonumber \frac{1}{T} \sum_{t=1}^T \mathbb{E} \lVert \nabla f(x_t) \rVert^2 \le \frac{C_1}{\sqrt{nT}} + \frac{C_1 + C_2}{T},
\end{equation}
where constants $C_1$ and $C_2$ are independent of $T$. 
\end{theorem}
Similar to the analysis of Algorithm~\ref{alg:parameter-average}, we also define $e_t = \frac{1}{n} \sum_{i=1}^n e_t^{(i)}$ and define an auxiliary sequence $\tilde{x}_t$ in the convergence analysis, which satisfies $\tilde{x}_1 = x_1$ and
\begin{align} \label{eq:40}
    \tilde{x}_{t+1} = \tilde{x}_t - \frac{1}{n} \sum_{i=1}^n \alpha_t \hat{v}_t^{-1/2} m_t^{(i)}.
\end{align}
We can prove that sequence $\tilde{x}_t$ satisfies the following Lemma~\ref{lemma3}
\begin{lemma} \label{lemma3}
In Algorithm~\ref{alg:gradient-average}, we always have
\begin{equation}
    x_t - \tilde{x}_t = \alpha_{t-1} \hat{v}_{t-1}^{-1/2} e_{t-1}.
\end{equation}
\end{lemma}
\begin{proof}
By the definition of $\tilde{x}_t$, $\Delta_t$ and $e_t$, we have
\begin{align}
     x_t - \tilde{x}_t = \alpha_{t-1} \hat{v}_{t-1}^{-1/2} e_{t-1} + \sum_{s=1}^{t-2} \alpha_s (\hat{v}_s^{-1/2} - \hat{v}_{s+1}^{-1/2}) e_s 
\end{align}
As $\hat{v}_s \ge v_s$ for each element, the coordinate in $v_{s+1}$ which is not updated at iteration $s+1$ keeps the same as $v_s$ and is always smaller that the corresponding coordinate in $\hat{v}_s$. Moreover, since $\hat{v}_{s+1} = \max \{ \hat{v}_s, v_{s+1} \}$, we reach the conclusion that for any index $j \notin \mathcal{I}_s$, the value of $j$-th coordinate in term $(\hat{v}_s^{-1/2} - \hat{v}_{s+1}^{-1/2})$ must be $0$. On the other hand, by the definition of $\Delta_t$ and $e_t$, for any index $j \in \mathcal{I}_s$, the $j$-th coordinate of $e_s$ is always $0$. Therefore, term $(\hat{v}_s^{-1/2} - \hat{v}_{s+1}^{-1/2})$ and $e_s$ are orthogonal and we can prove our Lemma~\ref{lemma3}.
\end{proof}
Next we will provide the proof outline of Theorem~\ref{thm2}.
\begin{proof}
 We define
 \vspace{-10pt}
\begin{equation}
    g_t = \frac{1}{n} \sum_{i=1}^n g_t^{(i)},\ m_t = \frac{1}{n} \sum_{i=1}^n m_t^{(i)}.
\end{equation}
It automatically satisfies
\begin{equation}
    \label{exp_avg}
    m_t = (1 - \beta_1) m_{t-1} + \beta_1 g_t,\ \tilde{x}_{t+1} = \tilde{x}_t - \alpha_t \hat{v}_t^{-1/2} m_t
\end{equation}
Let $A_t = \alpha_t \hat{v}_t^{-1/2} \nabla f(\tilde{x}_t)$ for $t=1, \cdots, T$ and $A_0 = A_1$. By Eq.~(\ref{exp_avg}) and $m_0 = \mathbf{0}$, it is easy to check
\begingroup
\small
\begin{align}
\label{OL-ga1}
    \sum_{t=1}^T \langle A_t, g_t \rangle &= \frac{\beta_1}{1 - \beta_1} \langle A_T, m_T \rangle + \sum_{t=1}^T \langle A_t, m_t \rangle \notag \\
    & \quad + \frac{\beta_1}{1 - \beta_1} \sum_{t=1}^T \langle A_t - A_{t+1}, m_t \rangle 
\end{align}
\endgroup
\vspace*{-10pt} \\
The left hand side of Eq.~(\ref{OL-ga1}) can be rewritten by
\begingroup
\small
\begin{align}
\label{OL-ga2}
    \nonumber \langle A_t, g_t \rangle &= \langle \alpha_{t-1} \hat{v}_{t-1}^{-1/2} \nabla f(x_t), g_t \rangle - \langle (\alpha_{t-1} \hat{v}_{t-1}^{-1/2} - \alpha_{t} \hat{v}_{t}^{-1/2})   \\
    & \cdot \nabla f(\tilde{x}_t), g_t \rangle - \langle \alpha_{t} \hat{v}_{t-1}^{-1/2} (\nabla f(x_t) - \nabla f(\tilde{x}_t)), g_t \rangle
\end{align}
\endgroup
Similar to Sketched-AMSGrad (GA), we want to obtain $\| \nabla f(x_t) \|^2$ by taking expectation on $g_t$. However, we cannot do this by taking expectation directly on $\langle A_t, g_t \rangle$ because $\hat{v}_t$ is also determined by $\xi_t^{(i)}$. But the previous value $\hat{v}_{t-1}$ does not depend on $\xi_t^{(i)}$. Therefore, we have
\begingroup
\small
\begin{align}
    \mathbb{E}_{\xi_t} \langle \alpha_{t-1} \hat{v}_{t-1}^{-1/2} \nabla f(x_t), g_t \rangle = \langle \alpha_{t-1} \hat{v}_{t-1}^{-1/2} \nabla f(x_t), \nabla f(x_t) \rangle
\end{align}
\endgroup
By Young's inequality and Assumption 3 we have
\begingroup
\small
\begin{equation}
\label{OL-ga1-T1}
\hspace{-15pt}
    \langle A_T, m_T \rangle \!=\! \langle \nabla f(\tilde{x}_T), \alpha_T \hat{v}_T^{-1/2} m_T \rangle \!\le\! L \lVert \alpha_T \hat{v}_T^{-1/2} m_T \rVert^2 \!+\! \frac{G^2d}{4L}
\end{equation}
\endgroup
The second term on the right of Eq.~(\ref{OL-ga1}) can be estimated by
\begingroup
\begin{align}
\label{OL-ga1-T2}
    \langle A_t, m_t \rangle &= \langle \nabla f(\tilde{x}_t), \alpha_t \hat{v}_t^{-1/2} m_t \rangle = \langle \nabla f(\tilde{x}_t), \tilde{x}_t - \tilde{x}_{t+1} \rangle \notag \\
    &\le f(\tilde{x}_t) - f(\tilde{x}_{t+1}) + \frac{L}{2} \lVert \tilde{x}_{t+1} - \tilde{x}_t \rVert^2
\end{align}
\endgroup
where the last inequality is due to Assumption 1. According to Assumptions 1 and 3, Eq.~(\ref{exp_avg}) and $\hat{v}_{t+1} \ge \hat{v}_t$, we can obtain
\begingroup
\allowdisplaybreaks
\begin{align}
\label{OL-ga1-T3}
\langle A_t - A_{t+1}, m_t \rangle &\le G^2(\lVert \alpha_t \hat{v}_t^{-1/2} \rVert_1 - \lVert \alpha_{t+1} \hat{v}_{t+1}^{-1/2} \rVert_1) \notag \\
& \quad + L \lVert \tilde{x}_{t+1} - \tilde{x}_t \rVert^2
\end{align}
\endgroup
Similarly, we can bound the second right term of Eq.~(\ref{OL-ga2})
\begingroup
\begin{align}
\label{OL-ga2-T2}
    & \quad \mathbb{E} \langle (\alpha_{t-1} \hat{v}_{t-1}^{-1/2} - \alpha_{t} \hat{v}_{t}^{-1/2}) \nabla f(\tilde{x}_t), g_t \rangle \notag \\
    &\le G^2 (\lVert \alpha_{t-1} \hat{v}_{t-1}^{-1/2} \rVert_1 - \lVert \alpha_{t} \hat{v}_{t}^{-1/2} \rVert_1)
\end{align}
\endgroup
Now we only need to estimate the last term of Eq.~(\ref{OL-ga2}). With probability $p > 1 - \delta$, it satisfies $\lVert \alpha_t e_t \rVert^2 \le (1 - \frac{k}{d})\ \lVert \frac{1}{n} \sum_{i=1}^n \alpha_t [\hat{v}_{t}^{(i)}]^{-1/2} m_t^{(i)} + \alpha_{t-1} e_{t-1} \rVert^2$. Otherwise with probability $p < \delta$, $\Delta_t$ is still some coordinates of $\tilde{\Delta}_t$. It always satisfies $\lVert \alpha_t e_t \rVert \le \lVert \frac{1}{n} \sum_{i=1}^n \alpha_t [\hat{v}_{t}^{(i)}]^{-1/2} m_t^{(i)} + \alpha_{t-1} e_{t-1} \rVert$. Hence we can get
\begingroup
\begin{align}
\label{err3}
\mathbb{E} \lVert \alpha_t \hat{v}_t^{-1/2} e_t \rVert^2 \le \gamma_1 \sum_{s=1}^t \gamma^{t-s} \mathbb{E} \lVert \alpha_s \hat{v}_s^{-1/2} m_s \rVert^2
\end{align}
\endgroup
Taking expectation, we have estimation
\begingroup
\small
\begin{align}
\label{OL-ga2-T3-s}
& \quad \mathbb{E} \langle \alpha_{t} \hat{v}_{t-1}^{-1/2} (\nabla f(x_t) - \nabla f(\tilde{x}_t)), g_t \rangle \notag \\
&= \mathbb{E} \langle \alpha_{t} \hat{v}_{t-1}^{-1/2} (\nabla f(x_t) - \nabla f(\tilde{x}_t)), \nabla f(x_t) \rangle \notag \\
&\le \frac{1}{2} \mathbb{E} \langle \alpha_t [\hat{v}_{t-1}^{(i)}]^{-1/2} \nabla f(x_t), \nabla f(x_t) \rangle + \frac{\alpha_t L^2}{2 \sqrt{\epsilon}} \mathbb{E} \| x_t - \tilde{x}_t \|^2
\end{align}
\endgroup
The inequality results from Cauchy-Schwartz inequality and Assumption 1. Sum the last term of Eq.~(\ref{OL-ga2-T3-s}) from $t=1$ to $T$ and we have
\begingroup
\small
\allowdisplaybreaks
\begin{align}
\label{OL-ga2-T3}
\hspace{-20pt}
    & \quad \sum_{t=1}^T \mathbb{E} \lVert x_t - \tilde{x}_t \rVert^2 = \sum_{t=1}^T \mathbb{E} \lVert \alpha_{t-1} \hat{v}_{t-1}^{-1/2} e_{t-1} \rVert^2 \\
    &\le \!\gamma_1 \sum_{t=1}^T \sum_{s=1}^{t-1} \gamma^{t-1-s} \mathbb{E} \lVert \alpha_s \hat{v}_s^{-1/2} m_s \rVert^2 \!\le\! \frac{\gamma_1}{(1 - \gamma)} \sum_{t=1}^T \mathbb{E} \lVert \alpha_t \hat{v}_t^{-1/2} m_t \rVert^2 \notag
\end{align}
\endgroup
Combine Eqs.~(\ref{OL-ga1}), (\ref{OL-ga2}), (\ref{OL-ga1-T1}), (\ref{OL-ga1-T2}), (\ref{OL-ga1-T3}), (\ref{OL-ga2-T2}), (\ref{OL-ga2-T3-s}) and (\ref{OL-ga2-T3}). Take expectation and we have
\begingroup
\small
\allowdisplaybreaks
\begin{align}
\label{OTL-ga-2}
\nonumber & \quad \frac{1}{2} \sum_{t=1}^T \mathbb{E} \langle \alpha_{t-1} \hat{v}_{t-1}^{-1/2} \nabla f(x_t), \nabla f(x_t) \rangle \\
&\le f(\tilde{x}_1) - f(\tilde{x}_{T+1}) + \frac{\beta_1 G^2 d}{4L(1 - \beta_1)} + (\frac{L}{2} + \frac{2\beta_1 L}{1 - \beta_1}) \notag \\
& \quad \cdot \sum_{t=1}^T \mathbb{E} \lVert \tilde{x}_t - \tilde{x}_{t+1} \rVert^2 + G^2 (\lVert \alpha_0 \hat{v}_0^{-1/2} \rVert_1 - \mathbb{E} \lVert \alpha_{T} \hat{v}_{T}^{-1/2} \rVert_1) \notag \\
& \quad \!+\! \frac{\beta_1 G^2}{1 - \beta_1} (\lVert \alpha_1 \hat{v}_1^{-1/2} \rVert_1 \!-\! \mathbb{E} \lVert \alpha_{T+1} \hat{v}_{T+1}^{-1/2} \rVert_1) \!+\! \frac{L}{2} \sum_{t=1}^T \mathbb{E} \lVert \alpha_{t} \hat{v}_{t}^{-1/2} g_t \rVert^2 \notag \\
& \quad + \frac{\alpha_t L^2 \gamma_1}{2\sqrt{\epsilon} (1 - \gamma)} \sum_{t=1}^T \mathbb{E} \lVert \alpha_t \hat{v}_t^{-1/2} m_t \rVert^2 
\end{align}
\endgroup
As $\hat{v}_{t+1} \ge \hat{v}_t$, we have
\begingroup
\small
\begin{align}
\label{ga-mg}
\nonumber & \quad \sum_{t=1}^T \lVert \alpha_t \hat{v}_t^{-1/2} m_t \rVert^2 = (1 - \beta_1)^2 \sum_{t=1}^T \lVert \sum_{s=1}^t \beta_1^{t-s} \alpha_t \hat{v}_t^{-1/2} g_s \rVert^2 \\
&= (1 - \beta_1)^2 \sum_{t=1}^T \sum_{s,j=1}^t \beta_1^{2t-s-j} \langle \alpha_t \hat{v}_t^{-1/2} g_s, \alpha_t \hat{v}_t^{-1/2} g_j \rangle  \\
&\le (1 - \beta_1) \sum_{t=1}^T \sum_{s=1}^t \beta_1^{t-s} \lVert \alpha_s \hat{v}_s^{-1/2} g_s \rVert^2 \le \sum_{t=1}^T \lVert \alpha_t \hat{v}_t^{-1/2} g_t \rVert^2. \notag
\end{align}
\endgroup
By Lemma~\ref{lemma6} (shown after the proof) we also know that
\begingroup
\small
\begin{align}
\label{g-nabla}
\mathbb{E} \lVert g_t - \nabla f(x_t) \rVert^2 = \mathbb{E} \lVert \frac{1}{n} \sum_{i=1}^n (g_t^{(i)} - \nabla f_i(x_t)) \rVert^2 \le \frac{G^2d}{n}.
\end{align}
\endgroup
According to Eqs.~(\ref{exp_avg}), (\ref{OTL-ga-2}), (\ref{ga-mg}), (\ref{g-nabla}) and Assumption 2 we can obtain
\begingroup
\small
\allowdisplaybreaks
\begin{align}
\label{OTL-ga-3}
\nonumber & \quad \frac{1}{2} \sum_{t=1}^T \mathbb{E} \langle \alpha_{t-1} \hat{v}_{t-1}^{-1/2} \nabla f(x_t), \nabla f(x_t) \rangle \\
&\le f(x_1) \!-\! f^* \!+\! \frac{\beta_1 G^2 d}{4L(1 - \beta_1)} + \frac{G^2 \alpha_t}{\sqrt{ \epsilon}(1 - \beta_1)} + C_0 \sum_{t=1}^T \mathbb{E} \lVert \alpha_t \hat{v}_t^{-1/2} g_t \rVert^2 \notag \\
&\le f(x_1) - f^* + \frac{\beta_1 G^2 d}{4L(1 - \beta_1)} + \frac{G^2 \alpha_t}{\sqrt{ \epsilon}(1 - \beta_1)} + \frac{2C_0 G^2d}{n\epsilon} \sum_{t=1}^T \alpha_t^2 \notag \\
& \quad + \frac{2C_0}{\epsilon} \sum_{t=1}^T \alpha_t^2 \mathbb{E} \lVert \nabla f(x_t) \rVert^2.
\end{align}
\endgroup
where $C_0 = \frac{(1 + \beta_1)L}{1 - \beta_1} + \frac{\alpha_t L^2 \gamma_1}{2\sqrt{\epsilon} (1 - \gamma)}$ is a constant as $0 < \alpha_t < \alpha$. The left side of Eq.~(\ref{OTL-ga-3}) can be lower bounded by
\begingroup
\small
\begin{equation}
    \label{OTL-g3-left}
    \sum_{t=1}^T \mathbb{E} \langle \alpha_{t-1} \hat{v}_{t-1}^{-1/2} \nabla f(x_t), g_t \rangle \ge \sum_{t=1}^T \frac{\alpha_t}{G} \mathbb{E} \lVert \nabla f(x_t) \rVert^2.
\end{equation}
\endgroup
Since $\alpha_t \le \frac{\epsilon}{8C_0 G}$ when $T$ is large, we have
\begingroup
\begin{align}
    \frac{\alpha}{4G\sqrt{1 + \frac{T}{n}}} &\sum_{t=1}^T \mathbb{E} \lVert \nabla f(x_t) \rVert^2 \le f(x_1) - f^* + \frac{\beta_1 G^2 d}{4L(1 - \beta_1)} \notag \\
    & \quad + \frac{G^2 \alpha_t}{\sqrt{\epsilon}(1 - \beta_1)} + \frac{2C_0 G^2d}{n\epsilon} \sum_{t=1}^T \alpha_t^2
\end{align}
\endgroup
Let
\begingroup
\begin{align}
C_1 &= \frac{4G(f(x_1) - f^*)}{\alpha} + \frac{\beta_1 G^3 d}{L \alpha (1 - \beta_1)} + \frac{8 (1 + \beta_1) L G^3 d \alpha}{\epsilon (1 - \beta_1)} \notag \\ 
C_2 &= \frac{4 L^2 G^3 d \alpha^2 \gamma_1}{\epsilon^{3/2} (1 - \gamma)} + \frac{4G^3}{\sqrt{\epsilon}(1 - \beta_1)} 
\end{align}
\endgroup
Then we can reach the conclusion
\begingroup
\begin{align}
\frac{1}{T} \sum_{t=1}^T \mathbb{E} \lVert \nabla f(x_t) \rVert^2 \le \frac{C_1}{\sqrt{nT}} + \frac{C_1 + C_2}{T}.
\end{align}
\endgroup
using the fact that $\sqrt{1 + x} \le 1 + \sqrt{x}$.
\end{proof}

\begin{lemma} \label{lemma6}
Let $X_1, \cdots, X_k$ be independent stochastic variables with $0$ means. Then we have $\ \mathbb{E} \lVert \sum_{j=1}^k X_j\rVert^2 = \sum_{j=1}^k \mathbb{E} \lVert X_j\rVert^2$.
\end{lemma}

\begin{figure*}[!ht]
  \centering
  \hspace{-2pt}
  \begin{subfigure}{0.32\linewidth}
    \includegraphics[width=\textwidth]{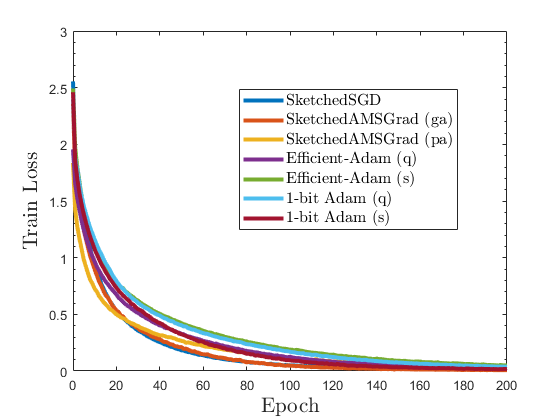}
  \end{subfigure}
  \hfill
  \begin{subfigure}{0.32\linewidth}
    \includegraphics[width=1.0\textwidth]{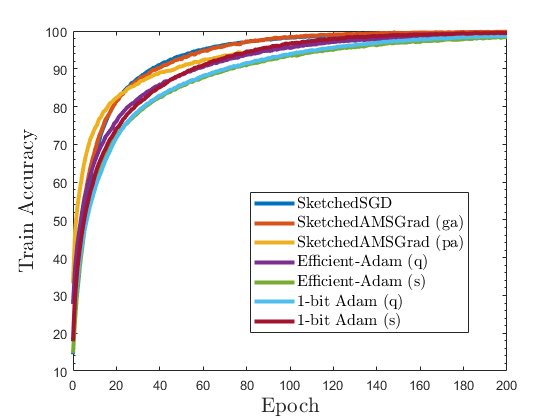}
  \end{subfigure}
  \hfill
  \begin{subfigure}{0.32\linewidth}
    \includegraphics[width=1.0\textwidth]{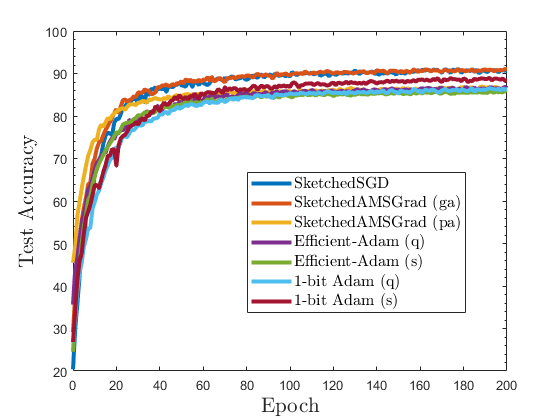}
  \end{subfigure}
  
  \begin{subfigure}{0.32\linewidth}
    \includegraphics[width=1.0\textwidth]{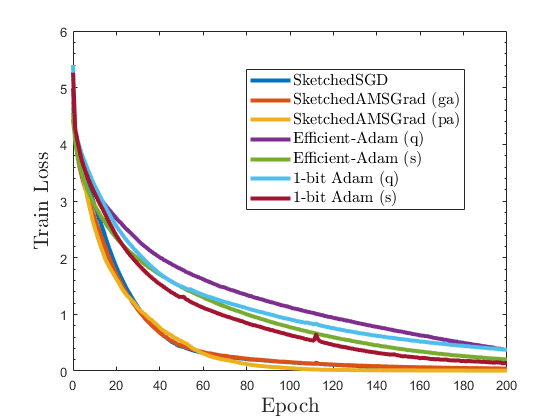}
  \end{subfigure}
  \hfill
  \begin{subfigure}{0.32\linewidth}
    \includegraphics[width=1.0\textwidth]{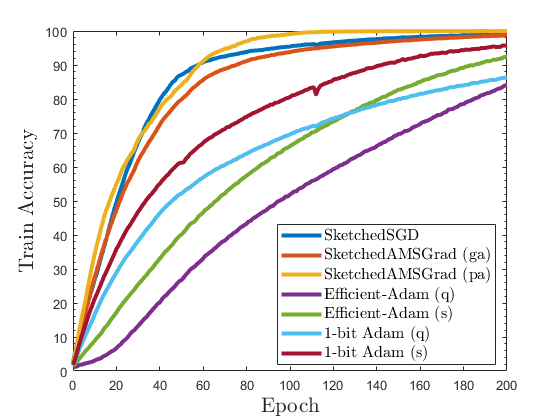}
  \end{subfigure}
  \hfill
  \begin{subfigure}{0.32\linewidth}
    \includegraphics[width=1.0\textwidth]{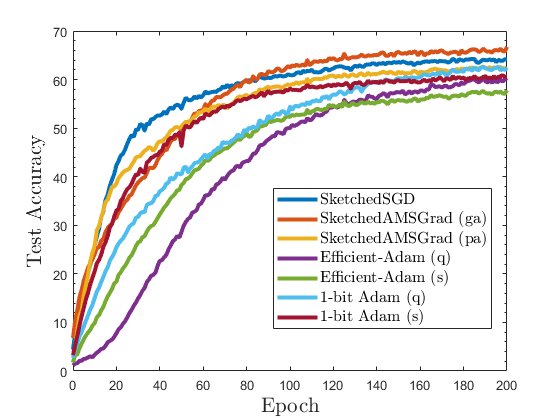}
  \end{subfigure}
  \caption{The experimental results of training ResNet-50 on CIFAR10 and CIFAR100. Figures (a), (b) and (c) show the experimental results on CIFAR10. Figures (d), (e) and (f) show the experimental results on CIFAR100. Figures (a) and (d) show the train loss value. Figures (b) and (e) show the train accuracy. Figures (c) and (f) show the test accuracy.}
  \label{resnet-cifar}
  \vspace{-10pt}
\end{figure*}
\vspace{-5pt}
\begin{corollary}
In Theorem~\ref{thm2}, we can see the dominating term is $O(\frac{1}{\sqrt{nT}})$, which achieves a linear speedup compared with AMSGrad in nonconvex optimization \cite{nonconvex-adam-1}.
\end{corollary}
\vspace{-5pt}
\begin{remark}
In Algorithm~\ref{alg:gradient-average}, the data distribution $D_i$'s are allowed to be non-identical, which implies it could be used in more general problems such as federated learning \cite{federate}. In both Algorithm~\ref{alg:parameter-average} and Algorithm~\ref{alg:gradient-average}, $\beta_1$ and $\beta_2$ are constants in $(0, 1)$, which is applicable to the common default settings that $\beta_1 = 0.9$ and $\beta_2 = 0.999$. 
\vspace{-0pt}
\end{remark}
\vspace{-0pt}
\subsection{Discussion on the Compression Rate}
\vspace{-0pt}
In this subsection, we will discuss the how the compression rate, \emph{i.e.}, the choice of $k$ influences the convergence rate. In both of Theorem~\ref{thm1} and Theorem~\ref{thm2}, constant $C_1$ is independent on $k$. Hence the dominating term is not affected by the compression rate. This result is the same as many other gradient compression methods. According to the definitions of $C_2$ in Theorem~\ref{thm1} and Theorem~\ref{thm2}, the second dominating term is affected by $k$ with the form $O(\frac{1}{(k/d)^2 T})$ for both parameter averaging and gradient averaging SketchedAMSGrad algorithms.
\vspace{-5pt}
\section{Experiments}
\vspace{-0pt}
\label{experiment}
In this section we will show the experimental results of two distributed data mining tasks of image categorization to validate our methods. All experiments are run on a server with 64-core Intel Xeon E5-2683 v4 2.10GHz processor and 4 Nvidia P40 GPUs. We simulate the edge-based training environment on the GPU server where the root process represents the edge server, each process represents an IoT device and the dataset represents the captured data. The codes are implemented in PyTorch 1.4.0 and CUDA 10.1.
\vspace{-0pt}
\subsection{ResNet on CIFAR} \label{resnet-cifar-task}
\vspace{-3pt}
Our first task is to train ResNet-50 \cite{resnet} using CIFAR10 and CIFAR100 datasets \cite{cifar}, which are benchmark datasets for image classification tasks. Both CIFAR10 and CIFAR100 contain 60,000 $32\times 32$ pixel images with RGB channels, 50,000 of which is regarded as training set and the other 10,000 of which is used for testing. The images are distributed evenly over 10 and 100 classes for CIFAR10 and CIFAR100 respectively. The ResNet-50 model has about 25M parameters. We use cross-entropy loss to train the neural network. 

In our experiment, we compare our SketchedAMSGrad (PA) and SketchedAMSGrad (GA) with Sketched-SGD \cite{sketched-sgd}, Efficient-Adam \cite{efficient-adam} and 1-bit Adam \cite{tang20211}. For Efficient-Adam and 1-bit Adam, we consider both quantization and sparsification as compressor. For gradient quantization, we adopt the following scheme used in \cite{dist-ef-sgd} which is a variant of SignSGD: 
\begingroup
\small
\begin{equation}
\label{compressor-exp}
    C(x) = \frac{\lVert x \rVert_1}{d} sign(x)
\end{equation}
\endgroup
Compressor~(\ref{compressor-exp}) automatically satisfies the compressor assumption Eq~(\ref{compressor}) and achieves about $32\times$ reduction of communication cost. For gradient sparsification, we use top-$k$ as the compressor.

\begin{figure*}[!t]
  \centering
  \begin{subfigure}{0.32\linewidth}
    \includegraphics[width=1.0\textwidth]{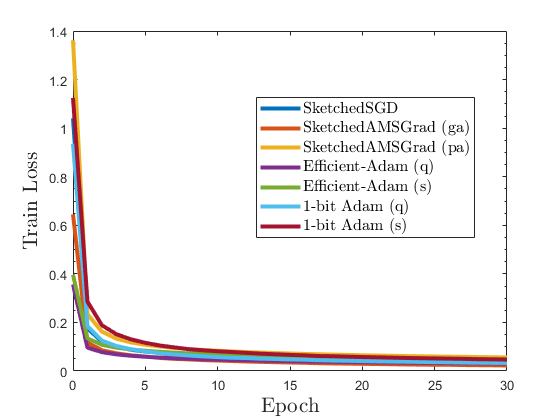}
  \end{subfigure}
  \hfill
  \begin{subfigure}{0.32\linewidth}
    \includegraphics[width=1.0\textwidth]{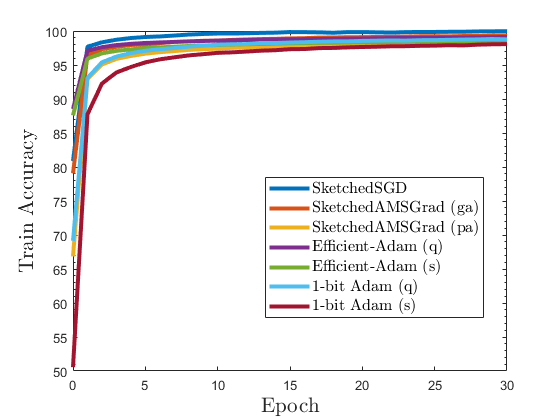}
  \end{subfigure}
  \hfill
  \begin{subfigure}{0.32\linewidth}
    \includegraphics[width=1.0\textwidth]{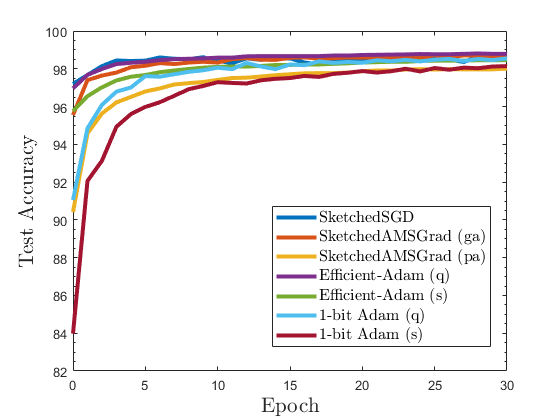}
  \end{subfigure}
  
  \begin{subfigure}{0.32\linewidth}
    \includegraphics[width=1.0\textwidth]{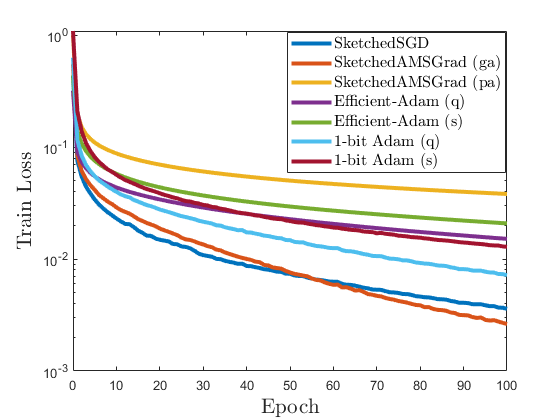}
  \end{subfigure}
  \hfill
  \begin{subfigure}{0.32\linewidth}
    \includegraphics[width=1.0\textwidth]{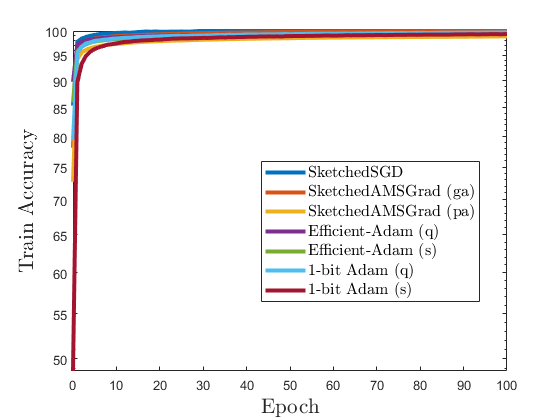}
  \end{subfigure}
  \hfill
  \begin{subfigure}{0.32\linewidth}
    \includegraphics[width=1.0\textwidth]{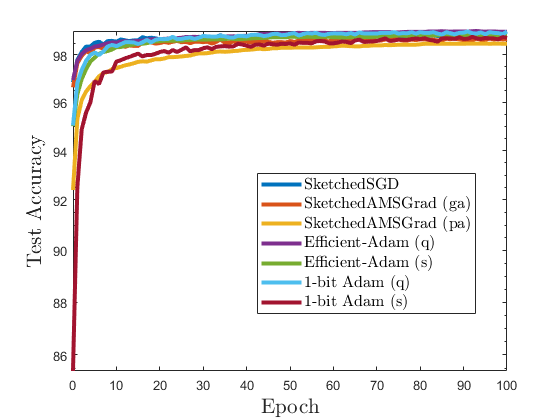}
  \end{subfigure}
  \caption{The experimental results of training LeNet-5 on MNIST. Figures (a), (b) and (c) show the experimental results when the number of workers is 50. Figures (d), (e) and (f) show the experimental results when the number of workers is 100. Figures (a) and (d) show the train loss value. Figures (b) and (e) show the train accuracy. Figures (c) and (f) show the test accuracy.}
  \label{lenet}
  \vspace{-12pt}
\end{figure*}

The number of workers in this task is set to be 16. The batch-size on each worker node is 32. Hence the total batch-size at each iteration is 512. We run 200 epochs in total. For each algorithm, we grid search the learning rate from $\{ 0.05, 0.01, 0.005, 0.001, 0.0005, 0.0001\}$ and $\epsilon$ from $\{ 1e-2, 1e-4, 1e-6 \}$ and select the values that get the best training result. For Adam-type algorithms, $\beta_1$ and $\beta_2$ are set to be the common choices that $\beta_1 = 0.9$ and $\beta_2 = 0.999$. For 1-bit Adam, similar to \cite{tang20211}, we run 13 epochs to compute the Adam-preconditioned vector $v_{T_w}$. For sketching methods, the sketch is set to have 100,000 columns and 10 rows. We set $k = 50,000$ and $P = 8$. For Efficient-Adam and 1-bit Adam with top-$k$ compressor, we choose $k = 750,000$. Therefore, all algorithms implemented in this task are communication-efficient and approximately achieve the same compression rate (about $32\times$ reduction).

Figure~\ref{resnet-cifar} shows the experimental results of this image classification task. According to the result of train loss value, we can see the three sketching methods converge faster than other algorithms on both CIFAR10 and CIFAR100 dataset. When comparing the train accuracy, the sketching methods are still advantageous over other methods. Our parameter averaging and gradient averaging SketchedAMSGrad and SketchedSGD approximately have the same performance. On CIFAR100, our parameter averaging SketchedAMSGrad is slightly better on the train accuracy results. According to the test accuracy results, our gradient averaging SketchedAMSGrad and SketchedSGD also outperform other algorithms on both dataset. On CIFAR100, our gradient averaging SketchedAMSGrad achieves the best performance on test accuracy. From this experiment we can see that although using compression on the returning message avoids the growing $O(n)$ communication cost issue of local top-$k$ (mentioned in \cite{sketched-sgd}), it probably encounters slow convergence since the estimator is too far away from the true top-$k$ coordinates. 

Theoretically, when the sketch size is larger, the probability of recovering top-$k$ coordinates is higher. The sketch size used in this experiment is 1,000,000. On CIFAR10, the test accuracy of our SketchedAMSGrad (GA) is 91.04\%. When we increase the sketch size to 2,000,000 and 3,000,000, the test accuracy is increased by 0.24\% and 0.39\% respectively. Thus, we can see the influence of sketch size. If the sketch size larger, our algorithm will probably show a better performance.
\vspace{-3pt}
\subsection{LeNet on MNIST}
\vspace{-3pt}
Our second task is to train MNIST dataset \cite{mnist} using LeNet-5 \cite{lenet}. This task is conducted in \cite{vrlsgd} under non-identical data partitioning. In this paper, we also run this experiment to verify the performance of our algorithms and related algorithms in the case of non-identical data distribution. MNIST is a database of hand written digits that is usually used for training image processing tasks. It contains 60,000 training images and 10,000 testing images from 10 classes. Each sample is a $28\times 28$ grayscale image. The training model used in this experiment is LeNet-5 which has about 60k parameters. We choose cross-entropy loss to be our criterion. The number of workers is set to be 50 and 100 respectively. Each worker can only access its local data and the data distribution is made non-identical. 

In this experiment we also compare our parameter averaging SketchedAMSGrad and gradient averaging SketchedAMSGrad with SketchedSGD, Efficient-Adam and 1-bit Adam. Both gradient quantization and gradient sparsification are considered as compressor in Efficient-Adam and 1-bit Adam. We also use compressor~(\ref{compressor-exp}) as quantization method and top-$k$ as sparsification method. We conduct two groups of experiments, with the number of workers $n = 50$ and $n = 100$ respectively. The total number of training epoch is 100. On each worker node, the batchsize is set to be 30. For each algorithm, we also grid search the learning rate from $\{ 0.05, 0.01, 0.005, 0.001, 0.0005, 0.0001\}$ and $\epsilon$ from $\{ 1e-2, 1e-4, 1e-6 \}$ and select the values that get the best training result. For Adam-type algorithms, we select $\beta_1 = 0.9$ and $\beta_2 = 0.999$ as usual. For 1-bit Adam, we also run 13 epochs to compute the Adam-preconditioned vector. For sketching methods, the sketch is set to have 400 columns and 5 rows. We set $k = 500$ and $P = 4$. For Efficient-Adam and 1-bit Adam with top-$k$ compressor, we choose $k = 2,000$. With these settings, the compression rate of different algorithms are approximately the same. The experimental results are shown in Figure~\ref{lenet}.

From the train loss results we can see that when data distribution is non-identical, generally gradient averaging algorithms SketchedSGD, SketchedAMSGrad (GA) and 1-bit Adam performs better than parameter averaging algorithms. Among these communication efficient adaptive gradient algorithms, our gradient averaging SketchedAMSGrad (GA) achieves the best convergence result in both $n = 50$ and $n = 100$ cases according to the train loss figures.
\vspace{-2pt}
\section{Conclusion} \label{conclusion}
In this paper, we propose a class of communication-efficient distributed adaptive gradient algorithm named SketchedAMSGrad based on two averaging strategies parameter averaging and gradient averaging to tackle the high communication cost issue for IoT edge-based training. Specifically, the communication cost of our algorithm at each iteration is reduced to $O(\log(d))$ from $O(d)$. 
Moreover, we proved that our algorithm achieves a fast convergence rate of $\tilde{O}(\frac{1}{\sqrt{nT}})$, which achieves the linear speedup with respect to the number of workers $n$, compared with single-machine AMSGrad. In particular, our analysis of gradient averaging SketchedAMSGrad can work for both identical and non-identical data distribution. To the best of our knowledge, our algorithm is the first to apply sketching technique to adaptive gradient methods.


\section*{Acknowledgment}

This work was partially supported by NSF IIS 1838627, 1837956, 1956002, 2211492, CNS 2213701, CCF 2217003, DBI 2225775.
\bibliographystyle{ieeetr}
\bibliography{icdm}


\end{document}